\documentclass[letterpaper]{article} 
\usepackage{aaai2026}  
\usepackage{times}  
\usepackage{helvet}  
\usepackage{courier}  
\usepackage[hyphens]{url}  
\usepackage{graphicx} 
\usepackage{natbib}  
\usepackage{caption} 
\usepackage{algorithm}
\usepackage{algorithmic}
\usepackage{textcomp}
\usepackage{amssymb}
\usepackage{amsmath}
\usepackage{amsthm}
\newtheorem{definition}{Definition}
\usepackage{booktabs}
\usepackage{multirow}
\usepackage{pifont}
\newtheorem{theorem}{Theorem}[section]
\newtheorem{lemma}[theorem]{Lemma}  

\usepackage{newfloat}
\usepackage{listings}
\DeclareCaptionStyle{ruled}{labelfont=normalfont,labelsep=colon,strut=off} 
\floatstyle{ruled}
\newfloat{listing}{tb}{lst}{}
\floatname{listing}{Listing}
\title{Synthetic Forgetting without Access: A Few-shot Zero-glance Framework for Machine Unlearning}
\author{
    Qipeng Song\textsuperscript{\rm 1},
    Nan Yang\textsuperscript{\rm 1},
    Ziqi Xu\textsuperscript{\rm 2},
    Yue Li\textsuperscript{\rm 1}\thanks{Corresponding author},
    Wei Shao\textsuperscript{\rm 3},
    Feng Xia\textsuperscript{\rm 2}
}
\affiliations{
    \textsuperscript{\rm 1}Xidian University\\
    \textsuperscript{\rm 2}Royal Melbourne Institute of Technology\\
    \textsuperscript{\rm 3}Data61, CSIRO\\

     \textbraceleft qpsong, Liyue\textbraceright@xidian.edu.cn, nanyang@stu.xidian.edu.cn\\
     \textbraceleft ziqi.xu,feng.xia\textbraceright@rmit.edu.au ,Phdweishao@gmail.com
}

\usepackage{bibentry}
\begin{document}

\maketitle

\begin{abstract}
Machine unlearning aims to eliminate the influence of specific data from trained models to ensure privacy compliance. However, most existing methods assume full access to the original training dataset, which is often impractical. We address a more realistic yet challenging setting: \textit{few-shot zero-glance}, where only a small subset of the retained data is available and the forget set is entirely inaccessible. We introduce GFOES, a novel framework comprising a Generative Feedback Network (GFN) and a two-phase fine-tuning procedure. GFN synthesises Optimal Erasure Samples (OES), which induce high loss on target classes, enabling the model to forget class-specific knowledge without access to the original forget data, while preserving performance on retained classes. The two-phase fine-tuning procedure enables aggressive forgetting in the first phase, followed by utility restoration in the second. Experiments on three image classification datasets demonstrate that GFOES achieves effective forgetting at both logit and representation levels, while maintaining strong performance using only 5\% of the original data. Our framework offers a practical and scalable solution for privacy-preserving machine learning under data-constrained conditions.
\end{abstract}


\section{Introduction}
\label{sec1}

In recent years, growing awareness of personal data rights has led to the introduction of various privacy regulations, including the General Data Protection Regulation (GDPR)~\cite{voigt2017eu}, the California Consumer Privacy Act (CCPA)~\cite{goldman2020introduction}, and China’s Personal Information Protection Law (PIPL)~\cite{calzada2022citizens}. These laws enshrine the \textit{Right to be Forgotten}~\cite{villaronga2018humans}, which grants individuals the right to request data deletion. For Machine Learning as a Service (MLaaS) providers, complying with this right requires not only permanently deleting personal data from storage systems, but also removing any knowledge derived from it and embedded in trained models. This challenge has given rise to a new research direction: \textit{machine unlearning}~\cite{cao2015towards}.


\begin{figure}[t]  
    \centering
    \includegraphics[width=0.46\textwidth]{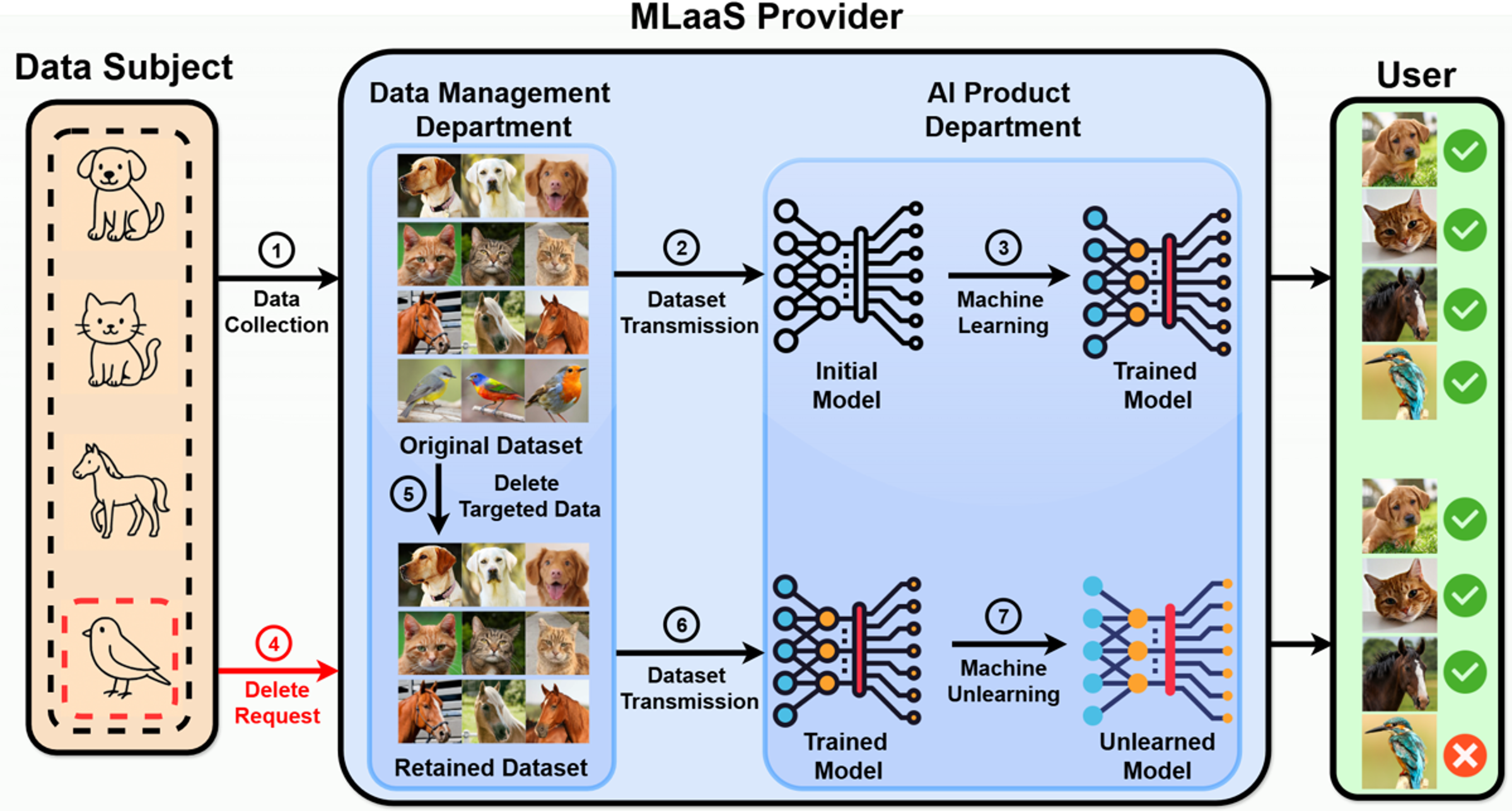}
    \caption{Overview of how an MLaaS provider processes data deletion requests under the \textit{Right to be Forgotten}. After data collection (\ding{172}) and initial model training (\ding{173}–\ding{174}), a deletion request (\ding{175}) triggers removal of the specified data (e.g., the bird image) from storage (\ding{176}). The updated dataset is used for unlearning (\ding{177}–\ding{178}), and the resulting model is then served to users.}
    \label{Fig0}
\end{figure}

To enforce the \textit{Right to be Forgotten}, MLaaS providers must ensure that both user data and its learned influence are thoroughly removed. As illustrated in Figure~\ref{Fig0}, the unlearning process involves two key objectives: (i) deleting the user's data from storage, and (ii) removing its impact from the trained model. Importantly, data deletion must precede unlearning to avoid unauthorised retention or reprocessing. This order aligns with Article 17(1) of the GDPR, which requires that personal data be erased \textit{without undue delay} once conditions such as consent withdrawal or data irrelevance are met. Performing unlearning before deletion risks caching or duplicating the data during intermediate steps, thereby violating GDPR principles like storage limitation and potentially leading to unlawful processing.

From both legal and technical perspectives, it is therefore essential that unlearning is performed \textit{after} irreversible data deletion. This order ensures that the learning system cannot access or replicate deleted data, thereby reinforcing privacy-preserving compliance and fostering trust in responsible AI systems. Furthermore, MLaaS providers are economically motivated to minimise the amount of data used in the unlearning process, as such compliance efforts incur operational costs without generating direct revenue. These constraints give rise to a new and practical unlearning setting, which we refer to as the \textit{few-shot zero-glance}: \textit{few-shot} reflects the minimal reliance on retained training data, while \textit{zero-glance} denotes the complete lack of access to the deleted data throughout the unlearning process.

Despite rapid progress in machine unlearning, existing methods fall short of meeting the stringent demands of the \textit{few-shot zero-glance} setting. Prior work generally falls into two categories: exact unlearning~\cite{wu2020deltagrad,yan2022arcane} and approximate unlearning~\cite{liu2022backdoor,shaik2023exploring,tanno2022repairing}. However, most methods in both categories assume full access to either the entire training set or the data to be forgotten. While a few methods~\cite{GuoGHM20} only consider the \textit{few-shot} setting, they still require access to the target data during unlearning and thus violate the \textit{zero-glance} setting. This highlights an urgent need for unlearning strategies that enforce both data minimisation and strict non-access guarantees, while preserving strong predictive performance.

To tackle the challenges of the \textit{few-shot zero-glance} setting, we propose a novel unlearning framework that removes class-specific knowledge without accessing the forget set. We first train a Generative Feedback Network (GFN) on a small subset of retained data to generate Optimal Erasure Samples (OES), which are synthetic instances labelled as the target class. These samples are crafted to interfere with forgotten-class knowledge while preserving the decision boundaries of retained classes. The GFN is trained with a stabilised joint objective that promotes forgetting through gradient ascent on target-class predictions and encourages retention through gradient descent on preserved-class outputs. To further enhance forgetting, we introduce a two-phase fine-tuning procedure: the first phase uses a large learning rate on both OES and retained data to overwrite class-specific representations, while the second phase uses a smaller learning rate and only retained data to refine decision boundaries and recover utility. 

Our main contributions are as follows:

\begin{itemize}
    \item We introduce GFOES, a practical framework for machine unlearning in the \textit{few-shot zero-glance} setting, enabling class-specific forgetting with no access to the forget set and minimal use of retained data.

    \item Our framework leverages synthetic OES and a two-phase fine-tuning procedure to balance effective forgetting and utility preservation, even under severe data constraints.

    \item Extensive experiments on CIFAR-10, CIFAR-100, and Fashion-MNIST demonstrate that GFOES consistently outperforms state-of-the-art methods in both forgetting effectiveness and model retention quality.
\end{itemize}

\section{Related Work}
\label{sec2}

\paragraph{Exact Unlearning} 
A key branch of machine unlearning research is exact unlearning, which focuses on replicating the outcome of retraining a model on a dataset with specific records removed, while reducing the computational cost. A foundational approach is the SISA framework~\cite{bourtoule2019machine}, which partitions training data into shards to enable selective retraining. This idea has inspired several extensions, including DaRE~\cite{brophy2021machine} for decision forests using cached statistics, GraphEraser~\cite{chen2022graph} for graph neural networks via graph partitioning, and asynchronous client-level unlearning in federated settings~\cite{su2023asynchronous}. Other strategies include memory-augmented retraining~\cite{yan2022arcane}, feature-score-based shortcuts~\cite{cao2015towards}, and Hessian-guided retraining~\cite{liu2022right}. Despite their theoretical appeal, exact unlearning methods remain impractical in real-world scenarios due to their computational and memory overhead.

\paragraph{Approximate Unlearning} 
To address the limitations of exact unlearning, approximate methods aim to remove the influence of target data without replicating full retraining. Gradient-based approaches include gradient ascent to erase backdoors~\cite{liu2022backdoor} and selective parameter tuning~\cite{fan2024salun}. Influence function-based methods, such as Certified Removal~\cite{GuoGHM20} and its scalable extensions~\cite{tanno2022repairing,suriyakumar2022algorithms}, estimate per-sample contributions. Hybrid strategies incorporate linear approximations~\cite{izzo2021approximate}, Fisher information masking~\cite{golatkar2021mixed}, and UNSIR~\cite{tarun2023fast}, which introduced the zero-glance setting. Despite their flexibility, these methods face two key challenges: (1) utility degradation due to catastrophic forgetting, and (2) reliance on auxiliary datasets to maintain accuracy~\cite{parisi2019continual,chundawat2023can}. While UNSIR operates without access to the forget set, our analysis indicates its limited effectiveness in removing representation-level knowledge under strict data constraints.

\paragraph{Few-shot Unlearning}
As a subset of approximate unlearning, few-shot unlearning removes the influence of target data using only a small portion of the original training set, reducing storage costs and reflecting realistic post-deletion scenarios~\cite{yoon2023fewshot}. Representative methods include model inversion for proxy data reconstruction~\cite{yoon2023fewshot} and manifold mixup for vertical federated learning~\cite{gu2024few}. A special case is zero-shot unlearning~\cite{fastchundawat2023zero,zhang2025toward}, which relies on knowledge distillation to unlearn without any original data. While zero-shot unlearning addresses a stricter challenging setting, it faces two key issues in practical few-shot scenarios: (1) lack of source data causes utility degradation even when some data remain; and (2) distillation often incurs high computational overhead, limiting scalability to large models and datasets.


In practice, MLaaS providers often retain a small portion of the original dataset, which, if effectively utilised, can significantly improve utility preservation during unlearning. These insights underscore the need for a framework that operates in the \textit{few-shot zero-glance} setting, achieving complete removal of class-specific knowledge while efficiently preserving utility with minimal retained data.


\section{Methodology}
\subsection{Problem Formulation}
We formalise the task of machine unlearning in the \textit{few-shot zero-glance} setting as follows: Consider a $K$-class classification task defined on a labelled dataset $\mathcal{D} = \{(x_i, y_i)\}_{i=1}^{n}$, where $x_i \in \mathcal{X} \subseteq \mathbb{R}^d$ represents the $d$-dimensional input features, and $y_i \in \mathcal{Y} = \{1, 2, \dots, K\}$ denotes the class labels. Let $M_0 = f(\cdot; \theta_0)$ be a predictive model trained on $\mathcal{D}$ by minimising a standard classification loss function $\mathcal{L}$. The optimisation objective is defined as: $\theta_0 = \text{argmin}_{\theta} \frac{1}{n} \sum_{i=1}^{n} \mathcal{L}(y_i, f(x_i; \theta)).$

Given a subset of categories to be forgotten, denoted as \( \mathcal{Y}_f \subset \mathcal{Y} \), we partition the original dataset \( \mathcal{D} \) into two disjoint subsets. The target dataset is defined as \( \mathcal{D}_f = \{(x_i, y_i) \mid y_i \in \mathcal{Y}_f\} \), which contains all instances belonging to the categories selected for forgetting. The retained dataset is defined as \( \mathcal{D}_r = \mathcal{D} \setminus \mathcal{D}_f \), which includes all remaining instances associated with the categories to be preserved.

In the ideal case where full access to \( \mathcal{D}_f \) and \( \mathcal{D}_r \) is available, the unlearning task can be modelled as a multi-objective optimisation:
\begin{equation}
\theta^* = \arg\min_{\theta} \left[ -\lambda \cdot \mathcal{L}(\mathcal{D}_f; \theta) + (1 - \lambda) \cdot \mathcal{L}(\mathcal{D}_r; \theta) \right],
\label{eq:ideal_unlearning}
\end{equation}
where \( \lambda \in (0,1) \) controls the trade-off between forgetting the target data and preserving the retained knowledge.

However, in the \textit{few-shot zero-glance} setting, two critical constraints apply:
\begin{itemize}
    \item \textbf{Zero-glance constraint:} The dataset \( \mathcal{D}_f \) targeted for removal is completely inaccessible due to privacy or regulatory requirements. As a result, \( \mathcal{D}_f \) cannot be directly used in the unlearning process.
    
    \item \textbf{Few-shot constraint:} Only a small subset of the retained dataset, denoted as \( \mathcal{D}_{rs} \subset \mathcal{D}_r \), is available for model adjustment, where typically \( |\mathcal{D}_{rs}| \ll |\mathcal{D}_r| \).
\end{itemize}

Consequently, direct optimisation of Eq.~\ref{eq:ideal_unlearning} becomes infeasible under the few-shot zero-glance constraints. To address this, we substitute \( \mathcal{D}_f \) with a set of synthetic samples \( \tilde{\mathcal{D}} = \{(\tilde{x}_i, \tilde{y}_i)\} \), where each \( \tilde{y}_i \in \mathcal{Y}_f \), and reformulate the unlearning objective as:
\begin{equation}
\theta^* = \arg\min_{\theta} \left[ -\lambda \cdot \mathcal{L}(\mathcal{Y}_f; f(\tilde{x}; \theta)) + (1 - \lambda) \cdot \mathcal{L}(\mathcal{D}_{rs}; \theta) \right],\nonumber
\label{eq3}
\end{equation}
where \( \tilde{x} \) denotes synthetic inputs representing the categories in \( \mathcal{Y}_f \), and \( \mathcal{D}_{rs} \subset \mathcal{D}_r \) is the few-shot retained subset available for model preservation.

Formally, machine unlearning under the \textit{few-shot zero-glance} setting aims to satisfy the following two objectives:

\begin{itemize}
    \item \textbf{Effective Unlearning:} The model should exhibit significantly degraded predictive accuracy on the forgotten classes. This is expressed as:
    \begin{equation}
        \mathbb{E}_{(x,y) \sim \mathcal{D}_f} \left[ \mathcal{L}(y, f(x; \theta^*)) \right] \gg \mathbb{E}_{(x,y) \sim \mathcal{D}_f} \left[ \mathcal{L}(y, f(x; \theta_0)) \right],\nonumber
        \label{eq:forgetting}
    \end{equation}
    where \( \theta_0 \) and \( \theta^* \) denote the model parameters before and after unlearning, respectively.

    \item \textbf{Performance Preservation:} The model should maintain predictive performance on the retained dataset:
    \begin{equation}
        \mathbb{E}_{(x,y) \sim \mathcal{D}_r} \left[ \mathcal{L}(y, f(x; \theta^*)) \right] \approx \mathbb{E}_{(x,y) \sim \mathcal{D}_r} \left[ \mathcal{L}(y, f(x; \theta_0)) \right].\nonumber
        \label{eq:retention}
    \end{equation}
\end{itemize}

This formulation encapsulates the core challenge of \textit{few-shot zero-glance} setting: constructing synthetic data that effectively induces forgetting while preserving model utility, all under strict data access constraints

\begin{figure*}[t]
    \centering
    \includegraphics[width=0.99\textwidth]{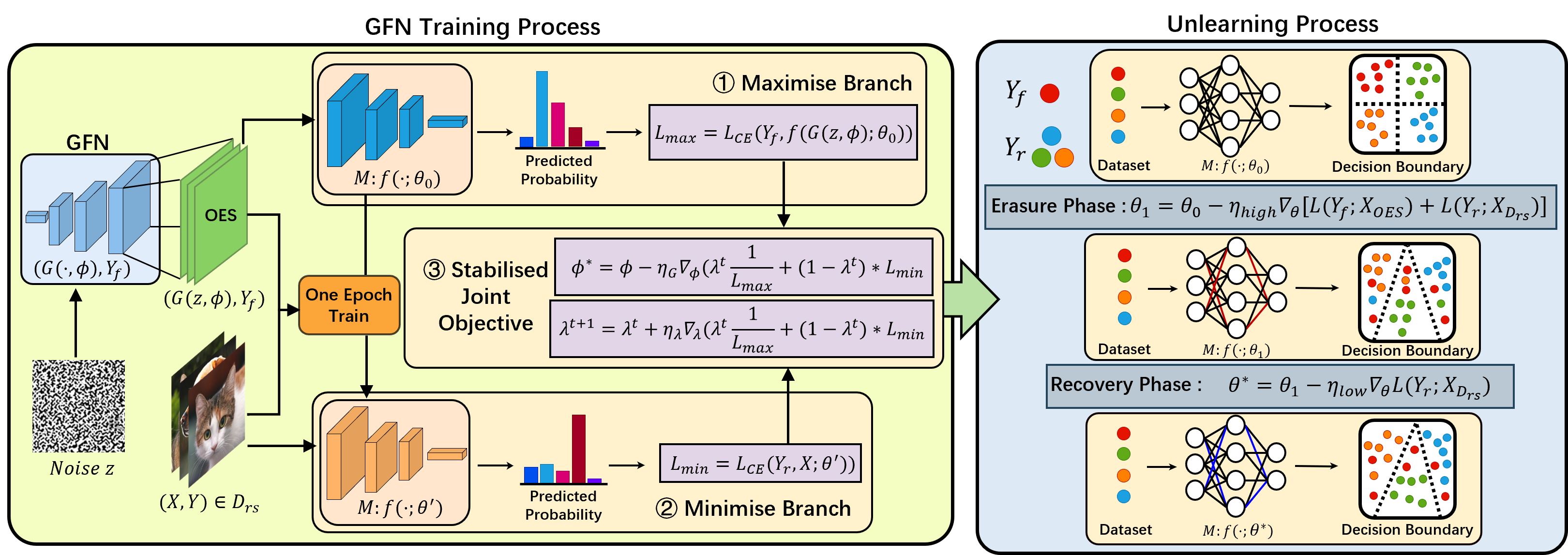}
    \caption{The training process of GFN consists of three components: the \textit{Maximise Branch}, the \textit{Minimise Branch}, and the \textit{Stabilised Joint Objective}, as illustrated on the left. The right side shows the machine unlearning procedure applied to the generated OES.}
    \label{Fig2}
\end{figure*}

\subsection{Optimal Erasure Samples}
\label{sec:OES}

Synthesising substitute samples \( \tilde{x} \) under the zero-glance constraint, where only the label set \( \mathcal{Y}_f \) of the forgotten classes is available, is inherently difficult because label information alone provides no access to the associated feature distributions. To address this, we adopt a reverse design strategy. Instead of approximating the true data distribution, we deliberately generate synthetic samples that deviate from it. These adversarial-like inputs, paired with the target labels, mislead the model into associating irrelevant features with the forgotten classes. As a result, the decision boundaries of the forgotten classes are disrupted with minimal interference to the retained ones.

We refer to these synthetic samples as {Optimal Erasure Samples (OES)}. Fine-tuning the model on OES weakens its ability to recognise the true semantics of classes in \( \mathcal{Y}_f \), thereby inducing effective class-level forgetting. Specifically, each sample from a class \( y \in \mathcal{Y}_f \) is replaced with an OES labelled as \( y \), while the original data for classes in \( \mathcal{Y}_r \) is preserved. This approach reduces the original multi-objective unlearning formulation (Eq.~\ref{eq:ideal_unlearning}) to a single fine-tuning objective as follows:
\begin{align}
&\theta^* = \arg\min_{\theta} \; \frac{1}{n} \sum_{i=1}^{n} \mathcal{L}\left(f(x_i'; \theta), y_i\right), \nonumber \\
& \quad \text{where} \quad
x_i' =
\begin{cases}
x^*, & \text{if } y_i \in \mathcal{Y}_f, \\
x_i, & \text{otherwise}.
\end{cases}
\label{eq:finetune_oes}
\end{align}

To construct effective OES under few-shot constraints, two key conditions must be met: (1) when labelled with \( \mathcal{Y}_f \), OES should yield high loss under the original model \( f(\cdot; \theta_0) \) and low loss under the unlearned model \( f(\cdot; \theta^*) \), disrupting the model's prior knowledge of the forgotten classes; (2) fine-tuning on OES should preserve accuracy on the retained set, keeping its loss low and comparable to pre-unlearning performance.

We now present a formal definition for OES that unifies these two constraints:

\begin{definition}[Optimal Erasure Samples (OES)]
Let \( \mathcal{D}^* = \{(x_i^*, y_i^*)\}_{i=1}^s \) be a synthetic dataset, where each \( y_i^* \in \mathcal{Y}_f \). This dataset qualifies as OES if it minimises the following composite objective:
\begin{align}
\mathcal{D}^* = & \arg\min_{\mathcal{D}'} \big[ 
     -\lambda \cdot \mathbb{E}_{(x, y) \in \mathcal{D}'} \mathcal{L}(y, f(x; \theta_0)) \nonumber \\
    & + (1 - \lambda) \cdot \mathbb{E}_{(x, y) \in \mathcal{D}_{rs}} \mathcal{L}(y, f(x; \theta^*)) \big],
\label{eq:oes_objective}
\end{align}
where \( \theta^* \) is a one-epoch updated model defined as:
\begin{equation}
\theta^* = \theta_0 - \eta \cdot \nabla_\theta \mathbb{E}_{(x, y) \in \mathcal{D}' \cup \mathcal{D}_{rs}} \mathcal{L}(y, f(x; \theta)) |_{\theta = \theta_0}.
\label{eq:theta_star}
\end{equation}
\end{definition}

Here, \( \lambda \in (0, 1) \) controls the trade-off between forgetting and retention of model utility. The first term encourages high loss on the forgotten classes to impair their influence, while the second term ensures that performance on the retained data is preserved following fine-tuning. The update step for \( \theta^* \) simulates the model's adaptation to the synthetic OES without requiring access to the actual data targeted for removal. This unified formulation captures both the destructive effect on the target classes and the stability requirement for the retained classes. 

\subsection{Generative Feedback Network}
\label{sec:gfn}

Constructing OES by directly solving the optimisation in Eq.~\eqref{eq:oes_objective} is computationally intractable. To overcome this challenge, we introduce a Generative Feedback Network (GFN), which transforms the OES search into a tractable, learnable optimisation process. As shown on the left side of Figure~\ref{Fig2}, the GFN decomposes the OES construction into two complementary sub-tasks: (1) maximising the classification loss on forgotten classes to induce forgetting, and (2) minimising the loss on retained classes to preserve utility. To achieve this, the architecture comprises three components: the \textit{Maximise} branch, the \textit{Minimise} branch, and a \textit{Stabilised Joint Objective} module that adaptively balances the two competing goals. These components jointly optimise a generator \( G(\cdot, \phi) \), initialised from random noise \( z \), to produce high-quality OES that meet both the forgetting and retention criteria in a stable and adaptive manner.

\subsubsection{Maximise Branch}
This component addresses constraint by encouraging the generator to produce synthetic samples that incur high classification loss under the original model \( f(\cdot; \theta_0) \). Formally, it optimises the following objective:
\begin{equation}
\mathcal{L}_{\max} = \mathcal{L}_{\mathrm{CE}}(\mathcal{Y}_f, f(G(z, \phi); \theta_0)),
\end{equation}
where \( G(z, \phi) \) denotes the synthetic sample generated from noise \( z \) via generator parameters \( \phi \), and \( \mathcal{L}_{\mathrm{CE}} \) is the cross-entropy loss computed with respect to the target labels \( \mathcal{Y}_f \).

\subsubsection{Minimise Branch} To preserve knowledge of the retained classes, we perform a one-epoch fine-tuning of the original model \( M_0 \) using both the generated samples and the few-shot retained dataset \( \mathcal{D}_{rs} \). The model parameters are updated as follows:
\begin{equation}
\theta' = \theta_0 - \eta_{\text{\tiny{GFN}}} \nabla_\theta \mathcal{L}(\mathcal{Y}_f \cup \mathcal{Y}_r, f(G(z, \phi) \cup \mathcal{D}_{rs}; \theta_0)),
\label{equal5}
\end{equation}
where \( \eta_{\text{\tiny{GFN}}} \) is the learning rate, and \( x_r \in \mathcal{D}_{rs} \). The updated model \( f(\cdot; \theta') \) is then evaluated on the retained classes to compute:
\begin{equation}
\mathcal{L}_{\min} = \mathcal{L}_{\mathrm{CE}}(\mathcal{Y}_r, f(x_r; \theta')).
\end{equation}

\subsubsection{Stabilised Joint Objective} The final training objective of GFN stabilises the conventional bi-objective loss by replacing \( \mathcal{L}_{\max} \) with its reciprocal:
\begin{equation}
\mathcal{L}_{\mathrm{GFN}} = \lambda \cdot \frac{1}{\mathcal{L}_{\max}} + (1 - \lambda) \cdot \mathcal{L}_{\min}, \quad \lambda \in (0,1).
\label{eq:gfn_final}
\end{equation}

This formulation mitigates the influence of large values in $\mathcal{L}_{\max}$ while retaining its gradient direction, as shown by its gradient: $\nabla_\phi \left( \frac{1}{\mathcal{L}_{\max}} \right) = -\frac{1}{\mathcal{L}_{\max}^2} \cdot \nabla_\phi \mathcal{L}_{\max}$.

To ensure both effective forgetting and utility preservation, we define a stabilised joint loss function for the generator:
\begin{equation}
\mathcal{L}_{\mathrm{GFN}}^{(t)} = \lambda_t \cdot \frac{1}{\mathcal{L}_{\max}} + (1 - \lambda_t) \cdot \mathcal{L}_{\min}, \quad \lambda_t \in (0,1),
\label{eq:gfn_adaptive}
\end{equation}
where the reciprocal form of \( \mathcal{L}_{\max} \) mitigates instability arising from unbounded gradients. The coefficient \( \lambda_t \) governs the trade-off between the forgetting and retention objectives at iteration \( t \), and is dynamically adjusted rather than manually selected.

To adaptively balance the competing objectives during training, \( \lambda_t \) is updated via gradient ascent on \( \mathcal{L}_{\mathrm{GFN}}^{(t)} \):
\begin{align}
&\lambda_{t+1} = \Pi_{(0,1)} ( \lambda_t + \eta \cdot \frac{\partial \mathcal{L}_{\mathrm{GFN}}^{(t)}}{\partial \lambda} ),\\
&\frac{\partial \mathcal{L}_{\mathrm{GFN}}^{(t)}}{\partial \lambda} = \frac{1}{\mathcal{L}_{\max}} - \mathcal{L}_{\min},
\end{align}where \( \eta > 0 \) denotes the learning rate, and \( \Pi_{(0,1)} \) is a projection operator ensuring that \( \lambda_t \) remains within the open interval \( (0,1) \). This mechanism obviates the need for manual hyperparameter tuning.

Under standard boundedness and differentiability assumptions, the update rule converges to a fixed point \( \lambda^* \) where \( \frac{1}{\mathcal{L}_{\max}(\lambda^*)} = \mathcal{L}_{\min}(\lambda^*) \), indicating a stable equilibrium between objectives. Formal analysis is provided in Appendix.

\subsection{Two-Stage Fine-Tuning for Unlearning}
\label{sec:two-stage-fine-tuning-unlearning}
Given the synthetic dataset $\mathcal{D}^*$ composed of OES, we adopt a two-phase fine-tuning procedure as shown on the right side of Figure~\ref{Fig2}, to balance effective forgetting and retention of useful knowledge. This strategy consists of two sequential phases: an \emph{Erasure Phase} that applies a large learning rate to jointly fine-tune the model on $\mathcal{D}^*$ and few-shot retained data $\mathcal{D}_{rs}$, followed by a \emph{Recovery Phase} that uses only $\mathcal{D}_{rs}$ with a smaller learning rate to restore the decision boundaries of preserved classes.

The strategy is motivated by the observation that aggressive updates are needed to disrupt representations of forgotten classes, yet they risk damaging the structure of retained ones. The second stage serves to repair this collateral damage without reintroducing the erased knowledge. While not grounded in formal theory, this design is empirically validated in Section~\ref{sec:ablation_study}, which shows its clear advantage over single-phase or uniformly trained alternatives.

\subsubsection{Erasure Phase}
This phase aggressively fine-tunes the model on both the synthetic OES targeting $\mathcal{Y}_f$ and the few-shot retained subset $\mathcal{D}_{rs}$, using a high learning rate $\eta_{\text{high}}$. The OES, generated adversarially, are designed to disrupt the model’s internal representation of the forgotten classes. Meanwhile, the inclusion of $\mathcal{D}_{rs}$ helps to stabilise optimisation and reduce collateral degradation.

The loss function is given by:
\begin{equation}
\mathcal{L}_{\text{erase}} = \mathcal{L}(\mathcal{Y}_f; f(\tilde{x}; \theta_0)) + \mathcal{L}(\mathcal{D}_{rs}; \theta_0),
\end{equation}
where $\tilde{x} = G(z, \phi^*)$ denotes synthetic inputs produced by the trained generator.

The model is then updated as:
\begin{equation}
\theta_1 = \theta_0 - \eta_{\text{high}} \cdot \nabla_{\theta} \mathcal{L}_{\text{erase}}.
\end{equation}

\subsubsection{Recovery Phase}
To mitigate any adverse impact on the retained classes introduced during the erasure phase, we perform a second fine-tuning stage using only the retained subset $\mathcal{D}_{rs}$ and a lower learning rate $\eta_{\text{low}}$. As the forgotten classes $\mathcal{Y}_f$ are excluded at this point, this stage serves to refine the model’s decision boundaries for $\mathcal{Y}_r$ without the risk of relearning discarded knowledge.

The model parameters are updated as:
\begin{equation}
\theta^* = \theta_1 - \eta_{\text{low}} \cdot \nabla_{\theta} \mathcal{L}(\mathcal{D}_{rs}; \theta_1).
\end{equation}

A detailed description of the procedure is presented in Appendix.

\section{Experiment}\label{sec4}
We conduct comprehensive experiments across multiple datasets, model architectures, and retained data ratios to evaluate the effectiveness of our method. Our validation assesses both the completeness of forgetting and the preservation of model utility using logit-based and representation-based metrics. We further perform ablation studies to examine the contribution of each key component. Due to space limitations, detailed configurations and experimental results are provided in Appendix. The code is available in the Supplementary Material.

\begin{table*}[t]
\centering
\footnotesize
\setlength{\tabcolsep}{4pt}
\scriptsize{
\begin{tabular}{l|l|ccc|ccc|ccc|ccc|ccc|ccc}
\toprule
& & \multicolumn{6}{c|}{Fashion-MNIST} & \multicolumn{6}{c|}{CIFAR-10} & \multicolumn{6}{c}{CIFAR-100} \\
& & \multicolumn{3}{c}{$\#y=1$} & \multicolumn{3}{c|}{$\#y=4$} & \multicolumn{3}{c}{$\#y=1$} & \multicolumn{3}{c|}{$\#y=4$} & \multicolumn{3}{c}{$\#y=1$} & \multicolumn{3}{c}{$\#y=10$} \\
& & 5\% & 10\% & 20\% & 5\% & 10\% & 20\% & 5\% & 10\% & 20\% & 5\% & 10\% & 20\% & 5\% & 10\% & 20\% & 5\% & 10\% & 20\% \\
\midrule
\multirow{11}{*}{\rotatebox{90}{$\mathcal{AD}_f$ (\%) $\downarrow$}}
& \textit{Original}  &86.80  &86.80  & 86.80 & 92.03 & 92.03 & 92.03 &85.82  &85.82  & 85.82  & 87.15 & 87.15  & 87.15 & 72.94 & 72.94 & 72.94 & 71.38 & 71.38 & 71.38 \\\cmidrule{2-20}
& Retrain & \textbf{0.00} & \textbf{0.00} & \textbf{0.00} & \textbf{0.00} & \textbf{0.00} & \textbf{0.00} & \textbf{0.00} & \textbf{0.00} & \textbf{0.00} & \textbf{0.00} & \textbf{0.00} & \textbf{0.00} & \textbf{0.00} & \textbf{0.00} & \textbf{0.00} & \textbf{0.00} & \textbf{0.00} & \textbf{0.00} \\
& NegGrad & 36.87 & 27.09 & 18.25 & 39.39 & 23.68 & 20.36  & 32.14 & 23.91 & 11.23 & 33.72 & 21.19 & 13.55 & 36.87 & 27.09 & 18.25 & 39.39 & 23.68 & 20.36 \\
& RL & 33.80 & 26.89 & 16.27 & 44.31 & 28.76 & 23.58 & 35.32 & 25.18 & 18.01 & 42.94 & 30.21 & 21.76  & 39.17 & 32.90 & 24.64 & 40.88 & 32.36 & 26.93 \\
& Fisher & 39.41 & 33.82 &27.65 & 41.29 & 35.37 & 30.08  & 40.73 & 32.15 & 26.09 & 42.87 & 34.01 & 28.46 & 32.41 & 23.37 & 17.02 & 33.46 & 26.78 & 19.12 \\
& UNSIR & \textbf{0.00} & \textbf{0.00} & \textbf{0.00} & \textbf{0.00} & \textbf{0.00} & \textbf{0.00} & \textbf{0.00} & \textbf{0.00} & \textbf{0.00} & \textbf{0.00} & \textbf{0.00} & \textbf{0.00} & \textbf{0.00} & \textbf{0.00} & \textbf{0.00} & \textbf{0.00} & \textbf{0.00} & \textbf{0.00} \\
& MI  & 25.47 & 13.85 & 7.12 & 27.33 & 13.92 & 6.41 & 27.51 & 12.09 & 6.02 & 27.82 & 13.45 & 9.55 & 29.81 & 12.89 & 8.92 & 26.12 & 14.85 & 10.05 \\
& GKT & \textbf{0.00} & \textbf{0.00} & \textbf{0.00} & \textbf{0.00} & \textbf{0.00} & \textbf{0.00} & \textbf{0.00} & \textbf{0.00} & \textbf{0.00} & \textbf{0.00} & \textbf{0.00} & \textbf{0.00} & \textbf{0.00} & \textbf{0.00} & \textbf{0.00} & \textbf{0.00} & \textbf{0.00} & \textbf{0.00} \\
& SalUn  & 25.81 & 13.66 & 6.92  & 23.44 & 16.39 & 4.27 & 27.92 & 11.41 & 8.36  & 21.13 & 17.82 & 2.94 & 29.31 & 13.28 & 9.67 & 22.76 & 15.54 & 5.11 \\
& GFOES & \textbf{0.00} & \textbf{0.00} & \textbf{0.00} & \textbf{0.00} & \textbf{0.00} & \textbf{0.00} & \textbf{0.00} & \textbf{0.00} & \textbf{0.00} & \textbf{0.00} & \textbf{0.00} & \textbf{0.00} & \textbf{0.00} & \textbf{0.00} & \textbf{0.00} & \textbf{0.00} & \textbf{0.00} & \textbf{0.00} \\
\midrule
\midrule
\multirow{11}{*}{\rotatebox{90}{$\mathcal{AD}_r$ (\%) $\uparrow$}}
& \textit{Original} &92.21  &92.21 & 92.21 & 91.42 & 91.42 & 91.42 &90.23   &90.23  & 90.23  & 91.56 & 91.56  & 91.56  & 71.47 & 71.47 & 71.47 & 71.49 & 71.49 & 71.49 \\\cmidrule{2-20}
& Retrain  &51.50  & 57.92 & 66.37 &50.49  & 58.20 & 64.03 &47.21  & 58.43 & 69.38 &52.92  & 59.22 & 68.06 & 46.44 & 49.07 & 53.50 & 47.23 & 50.95 & 54.24 \\
& NegGrad & 61.82 & 64.55 & 67.44 & 61.80 & 62.43 & 64.84 & 59.12 & 63.04 & 64.99 & 59.73 & 61.85 & 64.89 & 48.27 & 50.56 & 52.71 & 47.63 & 49.03 & 50.98 \\
& RL & 78.35 & 81.02 & 83.67 & 79.03 & 80.90 & 83.45 & 72.45 & 73.88 & 75.91 & 71.55 & 74.03 &77.04  & 58.01 & 60.15 & 62.76 & 57.41 & 59.00 & 61.94 \\
& Fisher & 61.45 & 67.89 & 73.21 & 60.56 & 65.78 & 71.12 & 62.34 & 63.45 & 69.56 & 61.23 & 66.78 & 71.23 & 47.89 & 54.56 & 59.34 & 49.01 & 56.78 & 59.43 \\
& UNSIR & 86.29 & 87.98 & 89.33 & 86.38 & 87.92 & 88.03 & 84.00 & 84.19 & 86.94 & 85.53 & 86.21 & 87.67  & 64.21 & 65.84 & 67.22 & 63.53 & 65.31 & 67.78 \\
& MI & 78.09 & 81.12 & 82.33 & 79.72 & 83.92 & 84.03 & 77.23 & 80.15 & 81.47 & 78.79 & 82.84 & 83.11  & 62.35 & 65.21 & 67.49 & 63.84 & 68.12 & 69.23 \\
& GKT & 87.41 & 87.41 & 87.41 & 86.91 & 86.91 & 86.91  & 55.23 & 55.23 & 55.23 & 48.76 & 48.76 & 48.76  & 46.17 & 46.17 & 46.17 & 30.29 & 30.29 & 30.29 \\
& SalUn  & 71.44 & 72.18 & 76.21 & 70.19 & 74.66 & 80.17  & 70.12 & 73.65 & 77.43 & 68.41 & 73.92 & 79.18 & 54.73 & 58.16 & 62.29 & 52.84 & 57.65 & 63.49 \\
& GFOES & \textbf{88.03} & \textbf{89.12} & \textbf{90.98} & \textbf{89.51} & \textbf{89.85} & \textbf{90.75} & \textbf{86.35} & \textbf{87.93} & \textbf{88.97} & \textbf{86.52} & \textbf{88.66} & \textbf{89.23} & \textbf{66.34} & \textbf{67.91} & \textbf{69.95} & \textbf{67.00} & \textbf{68.06} & \textbf{69.47} \\
\bottomrule
\end{tabular}
}
\caption{
Comprehensive results on three datasets across nine comparison methods and the proposed GFOES. $\uparrow$ indicates that higher values are better, while $\downarrow$ indicates that lower values are preferable. The \textit{Original} presents model performance prior to unlearning. $\#y$ denotes the number of classes to be unlearned, and $e$ represents the percentage of accessible training data. The best results are shown in \textbf{bold}.
}
\label{acc_all}
\end{table*}

\subsection{Experimental Setup}
\label{subsec_experiment_setup}
\subsubsection{Datasets and Model Architectures}  
We evaluate our framework on three image classification datasets: Fashion-MNIST~\cite{xiao2017fashion}, CIFAR-10, and CIFAR-100~\cite{krizhevsky2009learning}. In each setting, a subset of classes is randomly selected as the target for forgetting, and only 5\%, 10\%, or 20\% of samples per class are retained from the original training set. 

We adopt three representative models: 
(1) AllCNN~\cite{springenberg2014striving}, a lightweight convolutional network used for Fashion-MNIST; 
(2) ResNet-18~\cite{he2016deep}, a residual network with skip connections, applied to CIFAR-10; and 
(3) ResNet-50, a deeper residual architecture used for CIFAR-100.

\subsubsection{Comparison Methods}

As no existing method is explicitly designed for the \textit{few-shot zero-glance} setting, we evaluate a range of representative machine unlearning methods under this constraint. We include the following methods:  
(1) \textbf{Retrain}, a gold-standard baseline that trains a new model from scratch on the retained data only;  
(2) \textbf{NegGrad}~\cite{thudi2022unrolling}, which performs gradient ascent to increase the loss on forgotten data;  
(3) \textbf{RL}~\cite{golatkar2020eternal}, which assigns random labels to target samples for destructive fine-tuning;  
(4) \textbf{Fisher}~\cite{golatkar2020eternal}, which perturbs parameters based on their Fisher importance;  
(5) \textbf{UNSIR}~\cite{tarun2023fast}, a zero-glance method that injects adversarial noise to erase target representations;  
(6) \textbf{MI}~\cite{yoon2023fewshot}, which reconstructs proxy data via model inversion for fine-tuning;  
(7) \textbf{GKT}~\cite{fastchundawat2023zero}, a knowledge distillation-based method with thresholding to block target class transfer;  
(8) \textbf{SalUn}~\cite{fan2024salun}, which updates only saliency-sensitive parameters to forget specified data.

Specifically, Retrain, UNSIR, and GKT are naturally compatible with \textit{few-shot zero-glance} setting as they do not rely on access to the forgotten data. For the remaining methods that typically require access to the target data, we substitute the missing data with random noise inputs.

\subsection{Logit-based Evaluation}
\label{log-based}
Logit-based evaluation is a widely adopted protocol in machine unlearning~\cite{kim2025we}, facilitating an intuitive assessment of unlearning effectiveness. We employ two complementary metrics. The first is the \textit{accuracy on the forgotten dataset} ($A\mathcal{D}_f$), which quantifies the model's accuracy on data that ought to be forgotten; ideally, this value should approach zero, indicating successful erasure of target class knowledge. The second is the \textit{accuracy on the retained dataset} ($A\mathcal{D}_r$), which measures the model's accuracy on data that should be preserved. A higher value reflects better retention of the model’s original utility.

As shown in Table~\ref{acc_all}, GFOES consistently achieves $\mathcal{AD}_f = 0$ across all datasets and settings, indicating complete forgetting of the target classes. Comparable results are observed with Retrain, GKT, and UNSIR, which also attain zero forgetting error. However, only GFOES and UNSIR accomplish this without directly accessing the forget set, thereby satisfying the \textit{zero-glance} constraint. In contrast, methods such as NegGrad, RL, Fisher, MI, and SalUn rely on real data to compute gradients or saliency maps, and consequently struggle under the few-shot zero-glance setting. These methods exhibit substantial residual influence; for example, Fisher reaches $39.41\%$ on Fashion-MNIST with $\#y=1$ and $e=5\%$. Notably, SalUn and MI further degrade when $e$ is small, highlighting their sensitivity to noisy input and dependence on sufficient retained data.

In terms of retained utility, GFOES consistently achieves the highest $\mathcal{AD}_r$ across all datasets and settings. For example, on CIFAR-10 with $\#y=4$ and $e=20\%$, GFOES retains $89.23\%$ accuracy, surpassing RL ($77.04\%$), MI ($83.11\%$), and UNSIR ($87.67\%$). This strong performance is driven by two key components: an OES generator aligned with the retained class distribution and a two-phase fine-tuning strategy that separates forgetting from recovery. In contrast, Retrain underfits when data access is limited (e.g., $50.49\%$ on CIFAR-10, $\#y=4$, $e=5\%$), while GKT fails to generalise on complex datasets (e.g., $30.29\%$ on CIFAR-100, $\#y=10$, $e=20\%$). Other baselines show further drops due to their inability to cope with noisy or insufficient forget samples.

\subsection{Representation-based Evaluation}
\label{rep-based}
While logit-based evaluation is widely used to assess unlearning, it primarily reflects the model's reduced ability to classify forgotten classes and may not indicate whether internal representations related to those classes have been removed. Prior work~\cite{kim2025we} shows that when the feature extractor retains relevant knowledge, retraining the final layer can restore classification performance. To address this limitation, we include a representation-based evaluation that examines changes in the internal feature space of the model.



\begin{figure}[t]
\centering
 \includegraphics[width=0.48\textwidth]{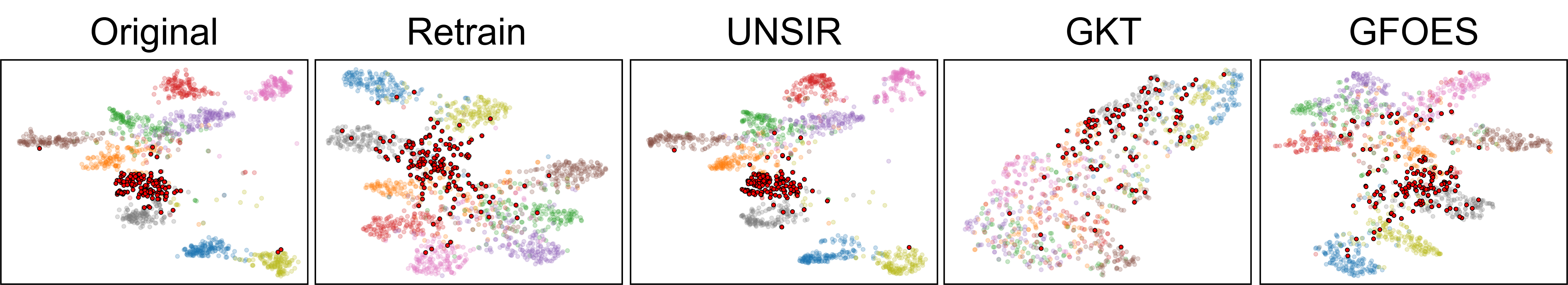}
    \caption{t-SNE visualisation of feature representations for single-class unlearning on CIFAR-10. Red points indicate samples from the forgotten class, while other colours represent the retained classes.}
\label{t-sine}
\end{figure}

\begin{figure}[t]
\centering
 \includegraphics[width=0.49\textwidth]{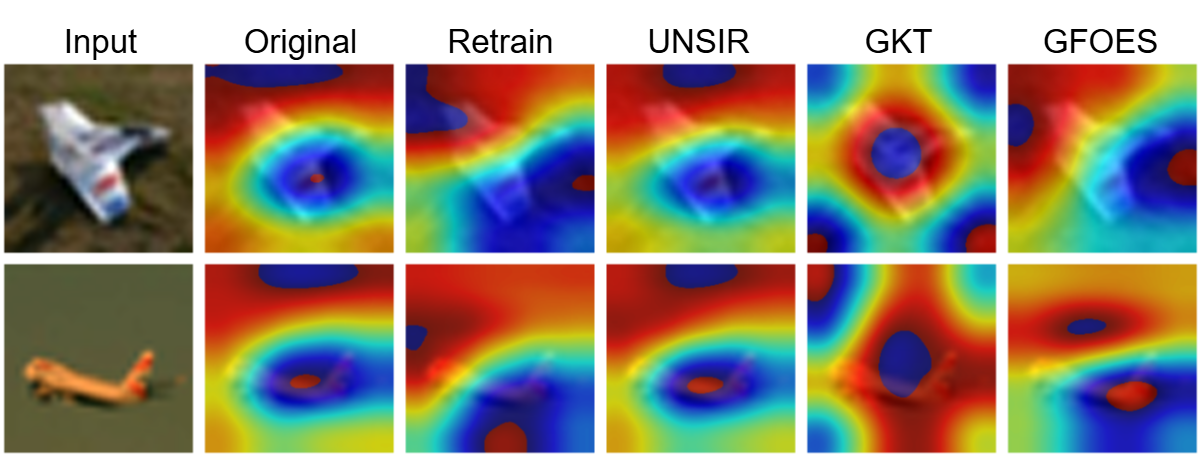}
\caption{GradCAM visualisation for single-class unlearning on CIFAR-10. Warmer colours highlight regions the model attends to when predicting the forgotten class, while cooler colours indicate lower attention.}
\label{gradcam}
\end{figure}

To evaluate whether unlearning disrupts class-specific representations, we apply t-SNE~\cite{van2008visualizing} to visualise feature extractor outputs on CIFAR-10 under the single-class setting. As shown in Figure~\ref{t-sine}, the original model forms well-separated clusters, while UNSIR preserves the forgotten class cluster, indicating incomplete forgetting. GKT scatters the forgotten class but also distorts retained ones, potentially harming utility. In contrast, our method disperses the forgotten class while maintaining the structure of retained classes, achieving targeted forgetting without compromising performance. As shown in Figure~\ref{gradcam}, GradCAM~\cite{selvaraju2017grad} further reveals that our framework, like Retrain, shifts attention away from semantically meaningful regions, whereas UNSIR retains focus on them. This confirms that GFOES erases both prediction and representation-level dependencies of the forgotten class.

\subsection{Time Efficiency Analysis}
Efficiency is essential for practical unlearning. To assess computational overhead, we report the wall-clock time of the unlearning process for {Retrain}, {GKT}, {UNSIR}, and {GFOES} across datasets in Table~\ref{time_cost_analysis_results}, as only these methods perform well in the \textit{few-shot zero-glance} setting. For {GFOES}, we report the time separately for GFN training, two-phase fine-tuning, and the overall total. As shown in Table~\ref{time_cost_analysis_results}, {GFOES} consistently incurs substantially lower time costs across all datasets and scenarios compared to {Retrain} and {GKT}, while maintaining high-quality unlearning. Although {GKT} supports zero-shot unlearning, its reliance on knowledge distillation leads to the highest overhead. {UNSIR} is the fastest, but it fails to remove feature representations (see Sec.~\ref{rep-based}), limiting its effectiveness. Overall, {GFOES} achieves an optimal balance between efficiency and unlearning quality, making it suitable for real-world deployment. Additional time efficiency analyses are provided in Appendix.

\begin{table}[t]
\centering
\setlength{\tabcolsep}{3pt} 
\scriptsize{
\begin{tabular}{l|l|cc|cc|cc}
\toprule
&\multirow{2}{*}{Method} & \multicolumn{2}{c|}{Fashion-MNIST} & \multicolumn{2}{c|}{CIFAR-10} & \multicolumn{2}{c}{CIFAR-100}  \\
\cmidrule(lr){3-4} \cmidrule(lr){5-6} \cmidrule(lr){7-8}
& & {$\#y=1$} & {$\#y=4$} 
& {$\#y=1$}  & {$\#y=4$}  
& {$\#y=1$}  & {$\#y=10$}  \\
\midrule
\multirow{6}{*}{\rotatebox{90}{Time (s) $\downarrow$}} 
& Retrain & 1457.63  & 1348.66 & 10816.30  & 8798.16 & 16198.26 & 14974.74 \\
& GKT  & 5420.43 & 4801.20 & 14760.69 & 12932.40 & 23623.65 & 21294.82 \\ 
& UNSIR & 15.52 & 13.46 & 26.68  & 23.69 & 42.82 & 38.42 \\  \cmidrule{2-8} 
& GFOES & 277.53 & 231.96 & 619.73 & 517.18 & 1165.78 & 1007.50 \\
& ~GFN & 263.42 & 218.05 & 597.56 & 498.80 & 1125.68 & 971.08 \\
& ~Two-Stage & 14.11 & 13.91 & 22.17 & 18.38 &  40.10 & 36.42 \\ 

\bottomrule
\end{tabular}
}
\caption{Wall-clock time (seconds) for unlearning methods. {GFN} denotes the time for generating Optimal Erasure Samples, {Two-Stage} for the fine-tuning phase, and {GFOES} for the total (GFN + Two-Stage).}
\label{time_cost_analysis_results}
\end{table}

\subsection{Ablation Study}
\label{sec:ablation_study}
To evaluate the contributions of core components in the proposed GFOES, we conduct an ablation study focusing on (i) the effectiveness of OES generation and (ii) the role of the two-phase fine-tuning procedure. For data composition, we compare {OES}, which uses only OES-generated samples during erasure and 10\% retained data in recovery, with {$\text{D}_{\text{r}}$}, which uses 10\% retained data in both phases. For two-stage fine-tuning procedure, we consider three strategies: {$\text{R}_{\text{ls}}$} (0.004 in erasure, 0.004 in recovery), {$\text{R}_{\text{l}}$} (0.004 in both), and {$\text{R}_{\text{s}}$} (0.0004 in both). Combining these options yields six configurations, which we evaluate on Fashion-MNIST, CIFAR-10, and CIFAR-100 under single-class unlearning setting to assess the impact of each component.

\begin{table}[t]
\centering
\setlength{\tabcolsep}{3pt}
\scriptsize{\begin{tabular}{l|cc|cc|cc}
\toprule
 \multirow{2}{*}{Setting} 
& \multicolumn{2}{c}{Fashion-MNIST} 
& \multicolumn{2}{c}{CIFAR-10} 
& \multicolumn{2}{c}{CIFAR-100} \\
\cmidrule(lr){2-3} \cmidrule(lr){4-5} \cmidrule(lr){6-7}

& $\mathcal{AD}_f$ $\downarrow$ & $\mathcal{AD}_r$ $\uparrow$ 
& $\mathcal{AD}_f$ $\downarrow$ & $\mathcal{AD}_r$ $\uparrow$ 
& $\mathcal{AD}_f$ $\downarrow$ & $\mathcal{AD}_r$ $\uparrow$ \\
\midrule
GFOES & \textbf{0.00} &89.45 & \textbf{0.00} & 87.54& \textbf{0.00} &68.72\\
~OES+$\text{D}_{\text{r}}$+$\text{R}_{\text{l}}$  & \textbf{0.00} & 62.32 & \textbf{0.00} & 58.66 & \textbf{0.00} & 39.12 \\
~OES+$\text{D}_{\text{r}}$+$\text{R}_{\text{s}}$  & 37.29 & 89.85 & 33.15 & 88.19 & 30.24 &  \textbf{70.54} \\
~OES+$\text{R}_{\text{ls}}$                 & \textbf{0.00} & 62.16 & \textbf{0.00} & 60.81 & \textbf{0.00} & 40.94 \\
~OES+$\text{R}_{\text{l}}$                  & \textbf{0.00} & 48.36 & \textbf{0.00} & 45.20 & \textbf{0.00} & 27.26 \\
~OES+$\text{R}_{\text{s}}$                  & 34.29 & 84.94 & 30.98 & 80.53 & 26.23 & 62.06 \\
~$\text{D}_{\text{r}}$+$\text{R}_{\text{ls}}$     & 15.19 & 79.88 & 11.30 & 77.16 & 12.67 & 58.93 \\
~$\text{D}_{\text{r}}$+$\text{R}_{\text{l}}$      & 4.30 & 68.36 & 5.68 & 65.01 & 6.37 & 46.03 \\
~$\text{D}_{\text{r}}$+$\text{R}_{\text{s}}$      & 42.74 &  \textbf{92.97} & 38.70 &  \textbf{88.82} & 34.99 & 70.05 \\
\bottomrule
\end{tabular}}
\caption{ Ablation results for single-class unlearning across three datasets. Each setting combines data composition ({OES} or $\text{D}_{\text{r}}$) and learning rate strategies ($\text{R}_{\text{ls}}$, $\text{R}_{\text{l}}$, or $\text{R}_{\text{s}}$). The best results are shown in \textbf{bold}.
}
\label{acc3}
\end{table}

The ablation results in Table~\ref{acc3} confirm the complementary roles of OES and the two-phase fine-tuning procedure. Configurations lacking OES (e.g., $\text{D}{\text{r}}$+$\text{R}{\text{l}}$, $\text{D}{\text{r}}$+$\text{R}{\text{s}}$) consistently fail to achieve full forgetting, indicating that retained data alone is insufficient to erase target class knowledge. Conversely, using OES without proper learning rate control (e.g., OES+$\text{R}{\text{l}}$ or OES+$\text{R}{\text{s}}$) leads to utility degradation or incomplete forgetting. Only the full GFOES configuration (OES+$\text{D}{\text{r}}$+$\text{R}{\text{ls}}$) consistently achieves optimal performance, balancing complete forgetting ($\mathcal{AD}_f = 0$) with minimal utility loss across all datasets. This unequivocally demonstrates that both synthetic erasure signals and staged fine-tuning are essential for effective and stable unlearning.


\section{Conclusion}
\label{sec5}
In this work, we propose GFOES, a novel machine unlearning framework tailored for the \textit{few-shot zero-glance} setting. GFOES introduces Optimal Erasure Samples as targeted adversarial signals to guide forgetting, combined with a two-phase fine-tuning procedure to remove class-specific knowledge while preserving performance on retained classes. Extensive experiments across diverse datasets and architectures demonstrate that GFOES achieves complete forgetting at both the logit and representation levels, while consistently outperforming existing baselines in utility retention under strict constraints. Ablation studies further confirm the importance of both the OES component and the two-phase fine-tuning procedure. Overall, GFOES provides a practical and scalable solution for data deletion compliance in real-world scenarios where direct access to the forget set is unavailable.

\bibliography{aaai2026}

\appendix
\setcounter{secnumdepth}{2}

\renewcommand{\thesection}{\Alph{section}} 
\renewcommand{\thesubsection}{\thesection.\arabic{subsection}}

\twocolumn[
\begin{center}
    \LARGE \textbf{Appendix for ``Synthetic Forgetting without Access: A Few-shot Zero-glance Framework for Machine Unlearning"}
\end{center}
\vspace{1em}
]

\section{Pseudo-code for the GFOES Framework}

This section presents pseudo-code for the two main components within the GFOES framework. Algorithm~\ref{alg:gfn} provides a description of the Generative Feedback Network (GFN), which synthesises Optimal Erasure Samples (OES) to support targeted forgetting under zero-glance constraints. Algorithm~\ref{alg:finetune} outlines the two-stage fine-tuning procedure, comprising an erasure phase that removes knowledge related to forgotten classes, followed by a recovery phase that restores model utility on the remaining classes.

\begin{algorithm}[h]
\caption{Generative Feedback Network (GFN)}
\label{alg:gfn}
\begin{algorithmic}[1]
\REQUIRE Initial model $f(x; \theta_0)$, retained data $\mathcal{D}_{rs}$, forgotten class labels $Y_f$, learning rate $\eta$, initial trade-off coefficient $\lambda_0$
\ENSURE Trained generator $G(z; \phi^*)$ 

\STATE Initialise generator parameters $\phi$ and set $\lambda \gets \lambda_0$
\FOR{$t = 1$ to $T$}
    \STATE Sample noise $z \sim \mathcal{N}(0, 1)$
    \STATE Generate synthetic samples  $G(; \phi)$
    
    \STATE \textbf{Maximise branch}:
    \STATE \quad Compute classification loss on forgotten classes:
    \begin{equation*}
    \mathcal{L}_{\max} = \mathcal{L}_{\mathrm{CE}}(\mathcal{Y}_f, f(G(z, \phi); \theta_0))
    \end{equation*}
    \STATE \textbf{One-epoch train}:
    \begin{equation*}
       \theta' \gets \theta_0 - \eta_{\text{\tiny{GFN}}} \nabla_\theta \mathcal{L}(\mathcal{Y}_f \cup \mathcal{Y}_r, f(G(z, \phi) \cup \mathcal{D}_{rs}; \theta_0))
    \end{equation*}
    \STATE \textbf{Maximise branch}:
    \STATE \quad Compute retained-class loss after one-epoch update: 
    \begin{equation*}
        L_{\text{min}} = \mathcal{L}_{CE}(Y_r, f(\mathcal{D}_{rs}; \theta'))
    \end{equation*}
    \STATE \textbf{Stabilised Joint Objective}:
    \STATE \quad Compute stabilised joint loss:
   \begin{equation*}
 \mathcal{L}_{\mathrm{GFN}}^{(t)} = \lambda_t \cdot \frac{1}{\mathcal{L}_{\max}} + (1 - \lambda_t) \cdot \mathcal{L}_{\min}
    \end{equation*}

    \STATE  \quad Update generator parameters: \\
    \begin{equation*}
    \phi^* \gets \phi - \eta \cdot \nabla_\phi \mathcal{L}_{\text{GFN}}
    \end{equation*}
    \STATE \quad Adapt trade-off coefficient:
    \begin{equation*}
        \lambda_{t+1} = \Pi_{(0,1)} ( \lambda_t + \eta \cdot \frac{\partial \mathcal{L}_{\mathrm{GFN}}^{(t)}}{\partial \lambda} )
    \end{equation*}
    \ENDFOR\RETURN Trained generator $G(; \phi^*)$
\label{alg_gafn}
\end{algorithmic}
\end{algorithm}

\begin{algorithm}[h]
\caption{Two-Stage Fine-Tuning}
\label{alg:finetune}
\begin{algorithmic}[1]
\REQUIRE Original model $f(x; \theta_0)$, trained generator $G(; \phi)$, retained data $\mathcal{D}_{rs}$, learning rates $\eta_{\text{high}}, \eta_{\text{low}}$
\ENSURE Unlearned model parameters $\theta^*$

\STATE \textbf{Erasure Phase}
\STATE \quad Sample latent code $z \sim \mathcal{N}(0, 1)$
\STATE  \quad Generate (OES): $\tilde{x}\gets G(z; \phi^*)$
\STATE  \quad Compute erasure loss: 
\begin{equation*}
\mathcal{L}_{\text{erase}} = \mathcal{L}(\mathcal{Y}_f; f(\tilde{x}; \theta_0)) + \mathcal{L}(\mathcal{D}_{rs}; \theta_0)
\end{equation*}
\STATE  \quad Update model: 
\begin{equation*}
\theta_1 \gets \theta_0 - \eta_{\text{high}} \cdot \nabla_{\theta} \mathcal{L}_{\text{erase}}
\end{equation*}
\STATE \textbf{Recovery Phase}
\STATE  \quad Compute recovery loss: 
\begin{equation*} 
\mathcal{L}_{\text{rec}} = \mathcal{L}_{CE}(Y_r, f(\mathcal{D}_{rs}; \theta_1))
\end{equation*}
\STATE  \quad Update model: 
\begin{equation*} 
\theta^* \gets \theta_1 - \eta_{\text{low}} \cdot \nabla_\theta \mathcal{L}_{\text{rec}}
\end{equation*}
\RETURN Unlearned model parameters $\theta^*$
\end{algorithmic}
\end{algorithm}

\section{Convergence Analysis of the GFN}

\subsection{Instability of the Subtractive Objective}

Let $F(\phi)$ and $G(\phi)$ denote the retention and forgetting losses, respectively. The conventional subtractive objective is defined as:
\begin{equation}
  \widetilde{\mathcal{L}}(\phi) = F(\phi) - G(\phi) \nonumber
\end{equation}
Although this formulation appears intuitive, it presents two main issues:
\begin{itemize}
    \item \textbf{Unboundedness}: As $G(\phi)$ can increase without limit while $F(\phi) \ge 0$, the objective becomes unbounded below, which may result in divergent optimisation paths.
    \item \textbf{Gradient instability}: The competing gradients $\nabla_\phi F$ and $\nabla_\phi G$ may differ in scale. When $G(\phi)$ dominates, the updates become excessively large, which causes oscillatory behaviour.
\end{itemize}

\subsection{Stabilised Joint Objective}

To address these issues, we introduce a stabilised joint objective that incorporates an adaptive balance coefficient \(\lambda_t \in (0, 1)\):
\begin{equation}
    \label{sec:appendix_stable_obj}
    \mathcal{J}_t(\phi_t, \lambda_t) = \lambda_t \frac{1}{G(\phi_t)} + (1-\lambda_t) F(\phi_t),~ \lambda_t \in (0,1).
\end{equation}

The parameters are updated through gradient descent with respect to \(\phi\) and projected gradient ascent with respect to \(\lambda\):
\begin{align}
    &\phi_{t+1} = \phi_t - \eta_\phi \nabla_\phi \mathcal{J}_t(\phi_t, \lambda_t), \\
&\lambda_{t+1} = \Pi_{[0,1]} \left( \lambda_t + \eta_\lambda \nabla_\lambda \mathcal{J}_t(\phi_t, \lambda_t) \right).
\end{align}

\subsection{Theoretical Assumptions}

We state the following mild assumptions:
\begin{itemize}
    \item \textbf{A1 (Smoothness)}: Both $F(\phi)$ and $G(\phi)$ are differentiable and $L$-smooth. That is,
    \[
    \|\nabla h(\phi) - \nabla h(\phi')\| \leq L \|\phi - \phi'\|, \quad \text{for all } h \in \{F, G\}.
    \]
    This condition holds for cross-entropy losses defined over bounded parameter domains.
    
    \item \textbf{A2 (Bounded domain)}: The feasible set $\Phi$ for $\phi$ is compact, which ensures that \(\sup_{\phi \in \Phi} G(\phi) < \infty\).
    
    \item \textbf{A3 (Positivity)}: For all \(\phi \in \Phi\), the forgetting loss satisfies \(G(\phi) \geq \epsilon > 0\). This property is natural in cross-entropy-based objectives.
\end{itemize}

\begin{table*}[t]
\centering
\begin{tabular}{l|ccccc}
\toprule
 & \textbf{Learning Rate} & \textbf{Epochs} & \textbf{Batch Size} & \textbf{Weight Decay} & \textbf{Gradient Clipping} \\
\midrule
AllCNN      & $4\mathrm{e}{-4}$ & 20  & 32 & $1\mathrm{e}{-4}$ & 0.1\\
ResNet-18   & $4\mathrm{e}{-4}$ & 100 & 32 & $1\mathrm{e}{-4}$ & 0.1\\
ResNet-50   & $4\mathrm{e}{-4}$ & 100 & 32 & $1\mathrm{e}{-4}$ & 0.1\\
\bottomrule
\end{tabular}
\caption{Training protocol for the original models.}
\label{tab:org}
\end{table*}

\begin{table*}[t]
\centering
\begin{tabular}{l|ccccc}
\toprule
 & \textbf{Learning Rate} & \textbf{Epochs} & \textbf{Batch Size} & \textbf{Weight Decay} & \textbf{Gradient Clipping} \\
\midrule
GFN             & $4\mathrm{e}{-3}$ & 20 & 32 & $1\mathrm{e}{-4}$ & 0.1\\
Erasure Phase   & $4\mathrm{e}{-3}$ & 1  & 32 & $1\mathrm{e}{-4}$ & 0.1\\
Recovery Phase  & $4\mathrm{e}{-4}$ & 1  & 32 & $1\mathrm{e}{-4}$ & 0.1\\
\bottomrule
\end{tabular}
\caption{Training protocol for the GFOES framework.}
\label{tab:GFOES}
\end{table*}

\subsection{Main Convergence Theorem}
The objective defined in Eq.~\eqref{sec:appendix_stable_obj} satisfies the following key properties:

\begin{lemma}[Lower bound]
Under Assumptions A2 and A3, the stabilised objective is strictly bounded below:
\begin{equation}
    \mathcal{J}_t(\phi_t, \lambda_t) \ge \frac{\lambda_t}{\sup_{\phi \in \Phi} G(\phi)} > 0. \nonumber
\end{equation}
\end{lemma}

\begin{proof}[Proof sketch]
Both terms in Eq.~\eqref{sec:appendix_stable_obj} are non-negative. The reciprocal term \(\lambda_t / G(\phi_t)\) is strictly positive due to A3, and the feasible domain ensures boundedness of \(G(\phi)\) from above by A2.
\end{proof}

\begin{lemma}[Gradient norm bound]
Suppose that \(F\) and \(G\) are \(L\)-smooth. Then the gradient norm of the objective with respect to \(\phi\) is bounded as:
\begin{equation}
    \|\nabla_{\phi_t} \mathcal{J}_t(\phi_t)\| 
    \le (1 - \lambda_t) L + \frac{\lambda_t L}{G(\phi_t)^2}. \nonumber
\end{equation}
This bound decreases as \(G(\phi_t)\) increases, mitigating potential gradient explosion.
\end{lemma}

\begin{proof}[Proof sketch]
Differentiating Eq.~\eqref{sec:appendix_stable_obj} with respect to \(\phi_t\) yields:
\[
\nabla_{\phi_t} \mathcal{J}_t(\phi_t) = (1 - \lambda_t) \nabla_\phi F(\phi_t) - \lambda_t \cdot \frac{\nabla_\phi G(\phi_t)}{G(\phi_t)^2}.
\]
Taking norms and applying the triangle inequality:
\[
\|\nabla_{\phi_t} \mathcal{J}_t\| \leq (1 - \lambda_t) \|\nabla_\phi F\| + \frac{\lambda_t}{G(\phi_t)^2} \|\nabla_\phi G\|.
\]
If both \(F\) and \(G\) are \(L\)-smooth, then their gradients are Lipschitz and satisfy:
\[
\|\nabla_\phi F\| \le L, \quad \|\nabla_\phi G\| \le L.
\]
Substituting gives:
\[
\|\nabla_{\phi_t} \mathcal{J}_t\| \le (1 - \lambda_t) L + \frac{\lambda_t L}{G(\phi_t)^2}.
\]
As \(G(\phi_t)\) increases, the second term decreases, which controls the gradient magnitude and stabilises updates.
\end{proof}

We now rigorously state and prove the convergence theorem for the stabilised GFN training process.

\begin{theorem}[Convergence of GFN Training]
\label{thm:convergence}
Consider the training of the Generative Feedback Network (GFN) with the stabilised objective defined in Eq.~\eqref{sec:appendix_stable_obj}, where the parameters \((\phi_t, \lambda_t)\) are updated via gradient descent in \(\phi\) and projected gradient ascent in \(\lambda\). Under Assumptions A1-A3, and with sufficiently small step sizes \(\eta_\phi, \eta_\lambda < 1/L\), the sequence of iterates \((\phi_t, \lambda_t)\) converges to a first-order stationary point. Specifically,
\[
\min_{1 \le t \le T} \left\| \nabla_\phi \mathcal{J}_t(\phi_t, \lambda_t) \right\|^2 
+ \left\| \nabla_\lambda \mathcal{J}_t(\phi_t, \lambda_t) \right\|^2 
\le \mathcal{O}\left(\frac{1}{T}\right).
\]
\end{theorem}

\begin{proof}
The optimisation problem defined by \(\mathcal{J}_t(\phi, \lambda)\) is a nonconvex-concave min–max problem. Due to neural network parameterisation and the reciprocal term \(1 / G(\phi)\), the objective is nonconvex in \(\phi\). Under Assumptions A1 and A3, where \(G(\phi) \ge \epsilon > 0\), the reciprocal term remains continuously differentiable and smooth. Thus, \(\mathcal{J}_t\) is a well-behaved, nonconvex smooth function in \(\phi\).

For fixed \(\phi\), the objective can be rewritten as
\[
\mathcal{J}_t(\phi, \lambda) = \lambda \left( \frac{1}{G(\phi)} - F(\phi) \right) + F(\phi),
\]
which is affine in \(\lambda\), ensuring concavity. Hence, the overall objective is a nonconvex–concave minimax problem.

From Lemma~2, we have the following uniform gradient bound:
\[
\|\nabla_\phi \mathcal{J}_t(\phi, \lambda)\| \le (1 - \lambda_t)L + \frac{\lambda_t L}{G(\phi)^2} \le L + \frac{L}{\epsilon^2},
\]
since \(G(\phi) \ge \epsilon > 0\). Together with Assumption A2, which ensures that \(\phi_t\) remains in a compact domain, this implies that the iterates \((\phi_t, \lambda_t)\) remain bounded and the training process is stable.

To quantify convergence, we define a Lyapunov-like scalar function:
\[
V(\phi, \lambda) = \frac{1}{2} \left( \|\nabla_\phi \mathcal{J}_t(\phi, \lambda)\|^2 + \|\nabla_\lambda \mathcal{J}_t(\phi, \lambda)\|^2 \right),
\]
which measures proximity to a stationary point. If \(V(\phi, \lambda) = 0\), then \((\phi, \lambda)\) is a first-order stationary point.

We now analyse the optimisation dynamics. From Assumption A1 (smoothness), the descent lemma implies that for \(\eta_\phi < 1 / L\), the gradient descent step in \(\phi\) satisfies:
\[
\mathcal{J}_t(\phi_{t+1}, \lambda_t) \le \mathcal{J}_t(\phi_t, \lambda_t) - \frac{\eta_\phi}{2} \|\nabla_\phi \mathcal{J}_t(\phi_t, \lambda_t)\|^2.
\]

Since \(\mathcal{J}_t\) is concave in \(\lambda\) and the domain \([0,1]\) is convex, the projected gradient ascent step ensures:
\[
\mathcal{J}_t(\phi_{t+1}, \lambda_{t+1}) \ge \mathcal{J}_t(\phi_{t+1}, \lambda_t) + \frac{\eta_\lambda}{2} \|\nabla_\lambda \mathcal{J}_t(\phi_{t+1}, \lambda_t)\|^2.
\]

Combining the two inequalities and summing over \(t = 1, \dots, T\), we obtain:
\begin{multline*}
\sum_{t=1}^T \left( \eta_\phi \|\nabla_\phi \mathcal{J}_t(\phi_t, \lambda_t)\|^2 
+ \eta_\lambda \|\nabla_\lambda \mathcal{J}_t(\phi_t, \lambda_t)\|^2 \right) \\
\le \mathcal{J}_t(\phi_1, \lambda_T) - \mathcal{J}_t(\phi_{T+1}, \lambda_1) \le M,
\end{multline*}
for some constant \(M > 0\), due to the boundedness of the objective.

Dividing both sides by \(T\) and applying the minimum inequality yields:
\begin{equation}
\begin{split}
\min_{1 \le t \le T} \left( \|\nabla_\phi \mathcal{J}_t(\phi_t, \lambda_t)\|^2 
+ \|\nabla_\lambda \mathcal{J}_t(\phi_t, \lambda_t)\|^2 \right) \\
\le \frac{M}{T \cdot \min(\eta_\phi, \eta_\lambda)} = \mathcal{O}\left(\frac{1}{T}\right).
\end{split}
\end{equation}

This result explicitly guarantees that the joint gradient norm converges to zero at a sublinear rate of \(\mathcal{O}(1/T)\). As \(T \to \infty\), the iterates \((\phi_t, \lambda_t)\) converge to a first-order stationary point. Hence, the training dynamics of the GFN under the stabilised objective are provably convergent.
\end{proof}

\section{Experimental Settings}
\label{appendix:exp-setup}

This section provides additional details regarding the datasets, model architectures, and training protocols. All experiments are implemented using PyTorch 2.4.1 and conducted on a single NVIDIA RTX 3090 GPU with 24\,GB of memory.

\subsection{Datasets}

We evaluate GFOES on three standard image classification benchmarks:

\begin{itemize}
    \item \textbf{Fashion-MNIST}: Consists of 60{,}000 grey-scale images (28\(\times\)28) across 10 categories. We randomly designate 1 or 4 classes as forget classes.
    
    \item \textbf{CIFAR-10}: Contains 60{,}000 RGB images (32\(\times\)32) across 10 object classes. In our setting, 1 or 4 classes are selected as forget classes.
    
    \item \textbf{CIFAR-100}: Comprises 60{,}000 RGB images (32\(\times\)32) from 100 fine-grained categories. We evaluate configurations with 1 or 10 forget classes.
\end{itemize}

For all datasets, the retained set is down-sampled to 5\%, 10\%, or 20\% of the original per-class instances to ensure class balance. The forget set remains entirely inaccessible during the unlearning process.

\begin{figure}[t]
\centering
\includegraphics[width=0.48\textwidth]{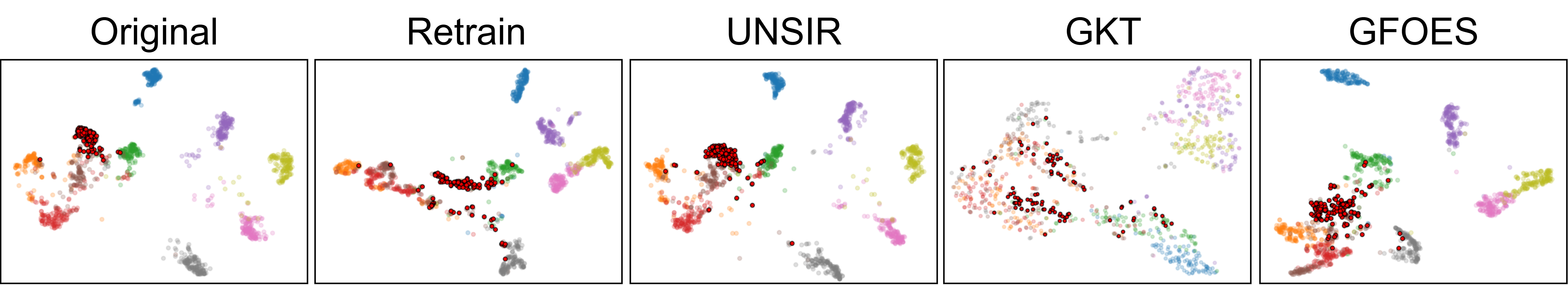}
\caption{t-SNE visualisation for single-class unlearning on Fashion-MNIST. GFOES disperses feature clusters associated with forgotten classes while preserving well-separated clusters for retained classes.}
\label{add-tsne}
\end{figure}

\subsection{Model Architectures}

We consider three widely used model architectures:

\begin{itemize}
    \item \textbf{AllCNN}: A fully convolutional network comprising 9 layers with ReLU activations. The number of channels increases from 96 (in the first three layers) to 192 (in the next three), followed by class-specific output layers. Downsampling is performed via stride-2 convolutions, and global features are aggregated using adaptive average pooling.
    
    \item \textbf{ResNet-18}: A lightweight residual network adapted for CIFAR inputs. The initial convolution uses a \(3 \times 3\) kernel (stride 1, without max pooling). The final fully connected layer is replaced with a dataset-specific linear classifier.
    
    \item \textbf{ResNet-50}: A deeper residual network initialised with \texttt{IMAGENET1K\_V2} pretrained weights. The initial stage mirrors ResNet-18, and the final layer is modified to output 100 logits for CIFAR-100.
\end{itemize}

\subsection{Training Protocols}

We standardise the training setup across all models and methods to enable fair comparison. Table~\ref{tab:org} summarises the training configurations used for the original models. For fairness, retraining procedures for all baseline methods follow identical settings. Table~\ref{tab:GFOES} details the hyperparameter settings for the GFOES framework. The erasure phase adopts a relatively large learning rate to aggressively remove the target knowledge, followed by a smaller rate in the recovery phase for stabilised fine-tuning. All baseline-specific hyperparameters are adopted directly from their original publications to ensure a fair comparison.

\begin{figure}[t]
\centering
\includegraphics[width=0.48\textwidth]{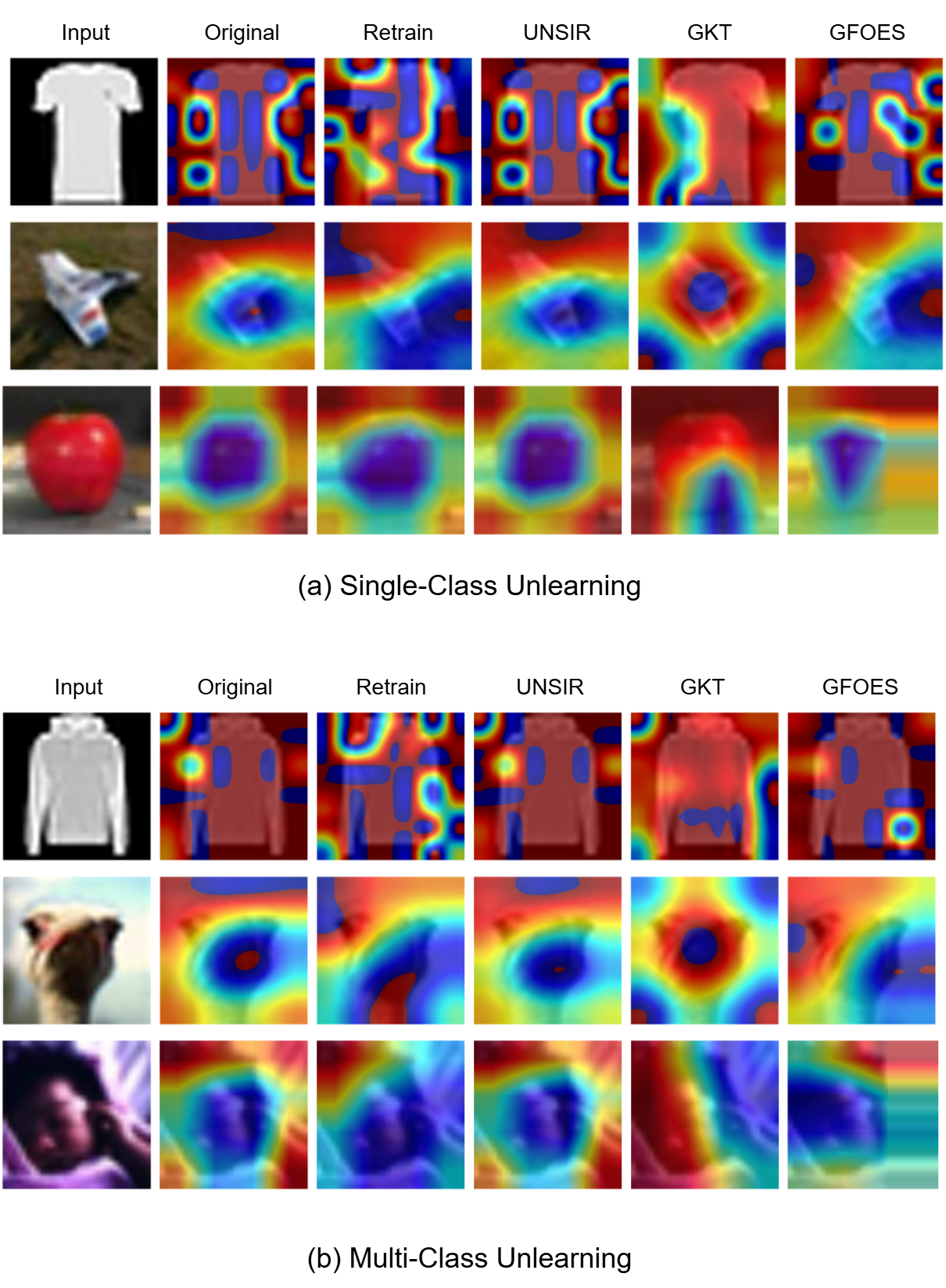}
\caption{GradCAM heatmaps. (a) Single-class unlearning results on Fashion-MNIST, CIFAR-10, and CIFAR-100; (b) Multi-class unlearning. Warmer colours indicate higher contributions to classification. GFOES eliminates reliance on forgotten-class regions while preserving attention on retained-class areas.}
\label{add-gradcam}
\end{figure}

\section{Additional Experimental Results}

\subsection{Feature Visualisation (Extended)}

In Section~4.3, we present feature visualisation results for single-class unlearning on CIFAR-10. Here, we provide extended results across multiple datasets and unlearning configurations. As shown in Figure~\ref{add-tsne} and Figure~\ref{add-gradcam}, GFOES consistently disrupts the feature structure of forgotten classes while maintaining coherent clusters for retained classes. GradCAM visualisations further confirm that GFOES significantly alters attention regions associated with forgotten classes, reinforcing its ability to remove feature-level knowledge without adversely affecting the representation of retained classes.

\subsection{Multi-class Ablation Study}

Table~\ref{add_ab} presents ablation results for multi-class unlearning. Consistent with the single-class findings, only the full GFOES configuration achieves \textit{complete forgetting} (\(\mathcal{AD}_f = 0\)) while maintaining competitive performance on \(\mathcal{AD}_r\). Other variants either underperform in forgetting or fail to preserve utility, confirming the necessity of both the OES mechanism and the two-stage learning rate schedule.

\begin{table*}[t]
\centering
\begin{tabular}{l|cc|cc|cc}
\toprule
\multirow{2}{*}{Setting} &
\multicolumn{2}{c}{Fashion} & \multicolumn{2}{c}{CIFAR-10} & \multicolumn{2}{c}{CIFAR-100} \\
\cmidrule(lr){2-3}\cmidrule(lr){4-5}\cmidrule(lr){6-7}
& $\mathcal{AD}_f\downarrow$ & $\mathcal{AD}_r\uparrow$ 
& $\mathcal{AD}_f\downarrow$ & $\mathcal{AD}_r\uparrow$ 
& $\mathcal{AD}_f\downarrow$ & $\mathcal{AD}_r\uparrow$ \\
\midrule
GFOES & \textbf{0.00} & \textbf{89.12} & \textbf{0.00} & \textbf{87.93} & \textbf{0.00} & \textbf{69.26}\\
~OES+$\text{D}_r$+$\text{R}_l$  & \textbf{0.00} & 58.74 & \textbf{0.00} & 60.81 & \textbf{0.00} & 38.29 \\
~OES+$\text{D}_r$+$\text{R}_s$  & 29.82 & 91.32 & 24.54 & 86.51 & 25.63 & 71.77 \\
~OES+$\text{R}_{ls}$ & \textbf{0.00} & 61.28 & \textbf{0.00} & 55.60 & \textbf{0.00} & 37.38 \\
~OES+$\text{R}_l$   & \textbf{0.00} & 42.83 & \textbf{0.00} & 40.34 & \textbf{0.00} & 30.92 \\
~OES+$\text{R}_s$   & 41.55 & 78.95 & 32.63 & 71.39 & 25.02 & 62.91 \\
~$\text{D}_r$+$\text{R}_{ls}$ & 24.26 & 76.90 & 17.73 & 81.26 & 18.92 & 51.29 \\
~$\text{D}_r$+$\text{R}_l$    & 12.39 & 73.12 & 7.90  & 68.66 & 8.99  & 52.89 \\
~$\text{D}_r$+$\text{R}_s$    & 42.74 & \textbf{92.97} & 38.70 & \textbf{88.82} & 34.99 & 70.05 \\
\bottomrule
\end{tabular}
\caption{ Ablation results for multi-class unlearning across three datasets. Each setting combines data composition ({OES} or $\text{D}_{\text{r}}$) and learning rate strategies ($\text{R}_{\text{ls}}$, $\text{R}_{\text{l}}$, or $\text{R}_{\text{s}}$). The best results are shown in \textbf{bold}.
}
\label{add_ab}
\end{table*}

\subsection{Weight Distance Analysis}
We conduct this analysis for both single-class and multi-class unlearning on three datasets. GFOES is compared with GKT and UNSIR, all achieving $A\mathcal{D}_f = 0$. To assess internal changes, we compute the L2 distance between the weights of the unlearned and original models, separating each into the feature extractor and final layer. A larger distance indicates greater deviation and stronger forgetting. As shown in Figure~\ref{weight1}, all methods significantly alter the final layer, disrupting classification. However, only GFOES and GKT also modify the feature extractor, indicating deeper removal of class-related knowledge. In contrast, UNSIR preserves this component, maintaining utility but failing to erase latent representations. These results suggest more comprehensive forgetting across both representation and prediction layers.


\begin{figure}[t]
\centering
\includegraphics[width=0.45\textwidth]{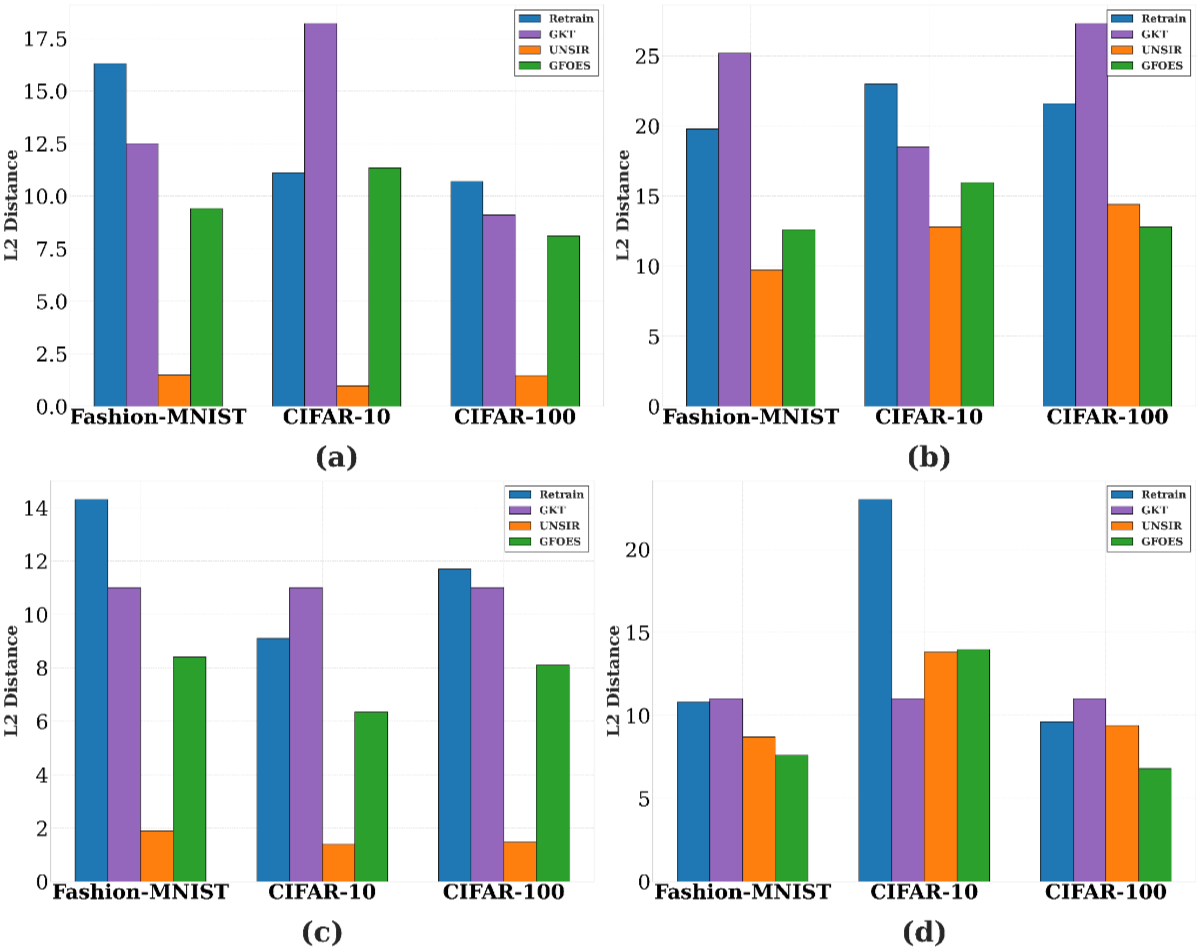}
\caption{L2 distance between original and unlearned models. (a–b) Single-class unlearning; (c–d) multi-class unlearning. (a, c) Feature extractor; (b, d) final classification layer. GFOES and GKT induce deeper changes at the feature level, whereas UNSIR primarily affects the classification head.}
\label{weight1}
\end{figure}

\subsection{Additional Time Efficiency Analysis}

Section~4.4 reports the wall-clock time for GFOES and baseline methods, demonstrating the superior efficiency of our approach. Here, we provide additional insights into the time cost breakdown and the design rationale underlying GFOES. As shown in Table~\ref{time_cost_analysis_results}, the majority of GFOES's time overhead arises from training the Generative Feedback Network (GFN). At first glance, one might expect that training a generator from scratch would be computationally expensive. However, this is not the case for our design.

The GFN in GFOES fundamentally differs from conventional generative models for realistic image synthesis. Instead of generating visually plausible samples, the GFN optimises a stabilised joint objective to produce pseudo-samples that (i) induce high loss on forgotten classes and (ii) minimise loss on retained classes under the unlearned model. In effect, it learns a structured noise distribution specifically tailored to erase target knowledge while preserving utility. This task-oriented formulation avoids the heavy computational demands typical of full image generation, such as those used in GANs or diffusion models.

As a result, GFOES achieves a favourable trade-off. While GFN training contributes the largest share of computation within GFOES, the total runtime remains an order of magnitude lower than that of retraining or distillation-based zero-shot methods such as GKT. These findings demonstrate that the additional generative step introduces minimal overhead while delivering substantial gains in utility preservation and unlearning quality.



\end{document}